\documentclass[accepted]{uai2022} % after acceptance, for a revised
\usepackage{gal_basic}

\usepackage{natbib} % has a nice set of citation styles and commands
\usepackage{mathtools} % amsmath with fixes and additions
\usepackage{booktabs} % commands to create good-looking tables
\usepackage{tikz} % nice language for creating drawings and diagrams

\newif\ifdraft
\draftfalse

\ifdraft 
\newcommand{\amirg}[1]{\textcolor{red}{AG: #1}}
\newcommand{\amiw}[1]{\textcolor{blue}{AW: #1}}
\newcommand{\gy}[1]{{\color{blue}[{\bf GY:}~#1]}}
\newcommand{\gale}[1]{{\color{magenta}[{\bf GE:}~#1]}}
\newcommand{\shay}[1]{{\color{blue}[{\bf Shay:}~#1]}}
\else
\newcommand{\amirg}[1]{} 
\newcommand{\amiw}[1]{} 
\newcommand{\gy}[1]{} 
\newcommand{\gale}[1]{} 
\newcommand{\shay}[1]{}
\fi

\usepackage{xr-hyper}

\title{Active Learning with Label Comparisons}

\author[1,2]{Gal Yona}
\author[1,3]{Shay Moran}
\author[1,4]{Gal Elidan}
\author[1,5]{\href{mailto:<amir.globerson@gmail.com>?Subject=Your UAI 2022 paper}{Amir Globerson}{}}

\affil[1]{%
    Google
}
\affil[2]{%
    Weizmann Institute of Science
}
\affil[3]{%
    Technion
}
\affil[4]{%
    Hebrew University
  }
 
\affil[5]{%
    Tel Aviv University
}

\begin{document}
\maketitle

\begin{abstract}
Supervised learning typically relies on manual annotation of the true labels. When there are many potential classes, searching for the best one can be prohibitive for a human annotator. On the other hand, comparing two candidate labels is often much easier. We focus on this type of pairwise supervision and ask how it can be used effectively in learning, and in particular in active learning. We obtain several insightful results in this context. In principle, finding the best of $k$ labels can be done with $k-1$ active queries. We show that there is a natural class where this approach is sub-optimal, and that there is a more comparison-efficient active learning scheme. A key element in our analysis is the ``label neighborhood graph'' of the true distribution, which has an edge between two classes if they share a decision boundary. We also show that in the PAC setting, pairwise comparisons cannot provide improved sample complexity in the worst case. We complement our theoretical results with experiments, clearly demonstrating the effect of the neighborhood graph on sample complexity.  

%\shay{Both the negative and positive theoretical results concern 1-dimensional thresholds, right? If so, I think a more precise interpretation is that we demonstrate theoretically a natural class in which pairwise comparisons do not provide improved statistical sample complexity, but they provide significant savings in query-complexity in the sense of active learning.}

%First, we show that in the standard PAC setting, even all pairwise comparisons are not more valuable than the true-label. On the other hand, for active learning, we show a way of using comparisons that is more sample effective than 
\end{abstract}

\section{Introduction}\label{sec:intro}
Supervised learning is a central paradigm in the empirical success of machine learning in general, and deep learning in particular. Despite the recent advances in unsupervised learning, and in particular self-training, large amounts of annotated data are still required in order to achieve high accuracy in many tasks. The main difficulty with supervised learning is, of course, the manual effort needed for annotating examples. Annotation becomes particularly challenging when there are many classes to consider. For example, in a text summarization task, we can ask an annotator to write a summary of the source text, but this will likely not result in the ``best'' summary. We could also present the annotator a summary and ask for feedback (e.g. is it good), but the quality could be difficult to judge in isolation. We could also ask the annotator to select the best summary out of a set of candidates (e.g. produced by a language model), but this could be taxing if not infeasible when there are many candidates.

%For example, consider collecting data for a dialogue agent. Asking an annotator to provide the best response at a given point in the conversation would require presenting all possible alternatives (potentially thousands), and asking the annotator to choose one among those. \gale{Many readers would instinctively feel that it is easy for a human to give a good answer. So, maybe take it more slowly here. Something like: "The number of possible replies can be huge, and even the number of good and bad response can be quite large at any given point in a conversation. In this scenario, selecting the single best answer can be quite prohibitive for a human labels." Do we want to give an example or is that overdoing it?} \gale{Consider also giving the example of a large number of object classes to give a sense of breadth of applicability}
%\shay{I also think we should give a more explicit example.}

Motivated by the above scenario, previous works \citep[e.g., see][for a recent application to large language models]{stiennon2020learning,ouyang2022training} have considered an alternative, and arguably natural, form of supervision: ``Label Comparisons''. Instead of presenting many potential labels to the annotator (e.g., candidate text summaries), we only present two candidates and ask the annotator to choose the better one. For example, when summarizing Snow White, we can ask to compare
the summaries ``A story about an evil step-mother'' and ``A story about a girl who is driven to the forest by an evil step-mother and ends up living with dwarves''. Most annotators would easily choose the latter as a better summary. 

%We refer to the above as learning with ``Label Comparisons''.

Label comparisons clearly require a much lighter cognitive load than considering all alternatives, and thus have high potential as an annotation mechanism. However, our theoretical understanding of this mechanism is fairly limited. While there has been work on learning to rank, which also uses comparisons, the goal of label comparisons is typically not to learn a complete ranking, but rather to build a model that outputs optimal predictions. Here we set out to analyze label comparisons from this perspective, and we obtain several surprising results and a new algorithm.
%Key questions in this context are: what is the sample complexity of a given label comparison scheme, what is the 
%they are not well understood from a theoretical perspective. For example, it is not known what is the sample 

%We note that our focus is very different from learning to rank, where the goal is to model a ranking function (e.g., over documents). Our goal here is the standard multi-class classification objective, and we only consider label comparisons as a means towards that end. Our work can be viewed as part of the larger effort of exploring supervision signals that are both realistic to obtain and informative (see \secref{sec:related}).
%\gy{Maybe say here that we assume there is a true (but uknown) total order on the classes, that defines both the argmax and the pairwise comparisons?}\amirg{Not sure we should go there at this point, since we are motivating the setup, and for humans it's not clear this holds.} 

%\gale{Is this the place to \emph{briefly} mention other works that used comparisons? We probably shouldn't keep it all to the related}

Our key question is what is the best way to learn with label comparisons. We assume 
that during learning we can only ask an annotator for label comparisons and not, for example, for the ground-truth label of the input, which we refer to as an argmax query. We then ask how one can design algorithms that make effective use of such queries, and what is the corresponding query complexity. Namely, how many queries are needed to achieve a given test error. Perhaps the most natural way of using comparisons is simply for finding the argmax label, which can be done via $k-1$ active queries. However, as we shall see, this is a suboptimal approach.

%The answer is not immediately clear, and we obtain several surprising results in this context.
%\shay{Consider presenting here the ``obvious'' benchmark algorithm of simulating an argmax query using $k-1$ comparisons. This naturally raises the question whether one can do better, and nicely prepares the ground for our results. }
%\gy{This paragraph is already talking about the active learning setup (implicitly), so it's a little confusing. I think I would start with the result for passive learning to motivate our focus on active learning, and then have this paragraph. I would also emphasize that a part of the challenge comes from the fact that label-comparisons introduce a second ``axis'' for the active learner (they can choose which x's and which i,j's to ask for).} 

The first question we ask is whether access to comparisons is more informative than access to the argmax. If we know all ${k \choose 2}$ comparisons for $x$, we can also infer the argmax and so it would seem like the answer to this should be in the affirmative.
%\gale{There was a comment here that 'we should say we assume total order' but this is in fact not needed for the argument at th e moment, right? Still, it might be useful to just go with total order from the get go}\amirg{we talk about argmax which implicitly assumes it's well defined. so yes I guess we should mention this somewhere.}.
Our first result shows that in the PAC setting, this is in fact not the case, and that knowing all comparisons may result in the same sample complexity as knowing only the argmax. The intuition for this negative result is that for 1D classifiers, the informative points are those that lie close to the decision boundaries between classes, and the argmax label for these points can also be used to find the boundaries, so that comparisons do not provide further advantage.
\gy{Changed decision boundary to decision boundaries (plural), so that it's clear we are talking about multiclass.}
\shay{It might be fair to note that this negative result really exploits the worst-case nature of PAC learning, and that it is certainly plausible that there are natural/practical learning tasks where comparisons provide insightful information. }

The negative result above may seem to suggest that comparisons are only useful for inferring the argmax. However, we show that in the case of active learning, comparisons can be used more effectively. We consider the setting where the active learner
can choose which label comparison queries to request for a
given input $x$ (including not requesting any queries at all). A natural approach here is to take a ``standard'' active learning algorithm based on argmax queries, and implement it using pairwise comparisons, by using $k-1$ active comparisons for each input $x$ to obtain the argmax. This strategy results in an algorithm that asks $\gamma (k-1)$ comparisons, where $\gamma$ is the number of argmax queries used. 

Here we show that one can in fact do better than simulating argmax, by asking the ``right'' comparisons in an active fashion. These beneficial comparisons are closely related to the ``Label Neighborhood Graph'' (see Figure \ref{fig:example_graph}) where labels are neighbors if they share a decision boundary. We show that it is sufficient to ask queries only about neighbor pairs in this graph. Thus, if this graph is sparse, active learning can be implemented with fewer queries. In particular, for linear classifiers in $\mathbb{R}$, each class has at most two neighbors, and thus the neighborhood graph is very sparse, and our proposed active learning approach is highly effective.
%\shay{The comment that $x\in\mathbb{R}$ is not clear to me. The reader cannot infer the class from this.}
%\gale{Consider walking through the 1-d case to illustrate the idea of how powerful a neighborhood graph can be}
%only \amirg{complete the complexity} label comparisons are required.
Taken together, our results demonstrate the richness of the label-comparison setting, and the ways in which its query complexity depends on the structure of the data. \gale{Say something about the results once they are in place. Also, it might be worthwhile to add to this first sentence the idea that in settings where the labler is limited (can't easily do argmax labeling), not much if lost if some a-priori knowledge is known about class relationship and give an example where it is natural to have such knowledge} 

\shay{Agreed, it would be nice if we can summarize the contribution in one-two sentences.
As far as I understand there are two main aspects: (i) argmax queries can be unrealistic in some applications whereas comparison queries are intuitive and useful in practice, 
(ii) while there is an obvious reduction whereby one simulates an argmax query using $k-1$ comparisons, there are natural learning tasks in which one can do significantly better by exploiting the geometry of comparisons more cleverly. }

\section{Related Work} \label{sec:related}

Several lines of works have addressed alternative modes of supervision for multi-class learning.

%\newline
{\bf Bandit Feedback:} In this setting (e.g., \cite{kakade2008efficient,crammer2013multiclass}) the learner only observes whether its predicted class is correct or not.
On the one hand, this feedback is stronger than label comparisons, because positive bandit feedback implies knowledge of the argmax. On the other hand, label comparisons provide more information than bandit feedback, because comparisons provide knowledge about relative ordering of non-argmax labels.

%\newline
{\bf Maxing from pairwise comparisons.}
 Maximum selection (maxing) from noisy comparisons is well-studied problem.  \cite{falahatgar2018limits} give an overview of known results under various noise models. Here, we show that for multiclass learning, using comparisons to first learn the global structure of the problem is more efficient than only using them for maxing. \cite{daskalakis2011sorting} consider maxing in partially ordered sets, where some pairs may be incomparable, which is interesting to explore in our setting.
 %. Here we focus on the case where all comparisons are available, but extension to posests is certainly interesting also in our setting.
 %which is a natural extension of our work since in practice we expect some labels to be incomparable. 
 %\vspace{-0.02in}

 %\newline
{\bf Dueling Bandits:} In online learning, learning from pairwise comparisons is studied under the dueling bandits setting
 \citep{saha2021adversarial, dudik2015contextual}, 
% , in both the stochastic \citep{} and adversarial \citep{saha2021adversarial} settings. Contextual extensions of the setting have also been studied  \citep{dudik2015contextual}. 
in which the learner ``pulls'' a pair of arms and observes the result of a noisy comparison (duel) between them. The objective in these cases is to minimize the regret w.r.t a solution-concept from the social choice literature, such as the Condorcet winner  \citep{yue2012k}, Borda winner, Copeland winner, or the Von Neuman winner \citep{dudik2015contextual}. The focus on such regret minimization objectives is principally different from ours, since our primary goal is to minimize the number of queries made, rather than minimizing an online loss.
%\vspace{-0.02in} 
 %\paragraph{Learning to rank:}

%\newline
{\bf Active Learning with rich supervision:} Several works have explored alternative forms of supervision. \cite{balcan2012robust} explore class-conditional queries, where the annotator is given a target label and a pool of examples, and must say whether one of the examples matches the target label. Several works \citep{kane2017active,hopkins2020power,xu2017noise} have studied comparison queries on instances, where the annotator receives two inputs $x_1,x_2$ and reports which one is more positive (for binary classification).
\cite{ben2022active} study active learning of polynomial threshold functions in $d=1$ using derivative queries (e.g., is a patient getting sicker or healthier?). 
Our supervision is conceptually different from all of these, as it compares between several labels on the same example $x$.
%\gale{Explicitly clarify that our multi class setting is different since this may be lost to the reader}
%\cite{kane2017active} show that for binary classification, allowing the learner to ask pairwise comparison queries on the \emph{instances} (e.g., given two patients, which one is more healthy?) provides exponential improvement over active learning with only label queries, and \cite{hopkins2020power} extend this to a distribution-specific setting. Similarly, \cite{xu2017noise} use both noisy labeling and pairwise comparisons.
%\vspace{-0.02in} 

{\bf Learning Ranking as a Reward Signal:} A recent line of work demonstrated that pairwise label-comparisons elicited from humans can be used to improve the performance of LLMs. \cite{stiennon2020learning} collect a dataset of human comparisons between summaries of a given text, and use it to obtain better summarization policies, and \cite{ouyang2022training} extend this idea to aligning LLMs with user intent. Our focus here is to understand the theoretical properties of such label comparisons, which we expect will result in more effective ways of collecting and using such comparisons.

%\gale{Explicitly say how we focus on a theoretical characterization of the setting to the difference will be clear}

\section{Preliminaries}
\label{sec:prelims}
\textbf{Multi-class learning}. Let $\X\subset \R^{d}$ denote the feature space and~$Y$ denote the label space, consisting of $k$ classes. We use $\D$ to denote an (unknown) distribution on $\X$ and $\H$ to denote a class of target functions, $f:\X\rightarrow \R^{k}$. In this work, our focus is on a realizable setting in which the target function is some (unknown) $f^\star \in \H$.
For $\xx \in \R^d$ and a class $i \in [k]$, $f_i(\xx)$ is the score assigned to class $i$ on instance $\xx$. Given a target function $f^\star$, the loss of a candidate classifier $f$ is the standard (multiclass) 0-1 loss: $   L(f) = \Pr_{\xx \sim \D}[\arg\max_{i \in [k]}f_i(\xx) \neq \arg\max_{i \in [k]}f_i^\star(\xx)]$.

Since we are interested in how the difficulty of learning scales with the number of classes $k$, we will explicitly parameterize hypothesis classes in terms of $k$, $\set{\H^k}_{k \in \N}$. For example, the class of homogeneous linear classifiers\footnote{Our convention will be to use homogeneous linear classifiers. Thus when we refer to our results for 1d, we mean the class $\H_{\linear}^{k,2}$.} over  $k \in \N$ classes in dimension $d$ is $\H_{\linear}^{k,d}=\left\{ h(\cdot;\WW): \WW \in \R^{k\times d}\right\}$, where  $h(\xx;\WW) = \WW \xx \in \R^{k}$. %\amirg{I prefer $h(\xx;\WW)$}

%\amirg{regarding this paragraph: it will not be clear to the reader. Maybe at this point we can just say that it's equivalent to non-homogeneous in 1D, which makes it easy to plot things. The center view can be justified later when it's needed.} 

\textbf{Supervision Oracles. }%\amirg{maybe title "Supervision Oracles"} 
Pertinent to this work is a distinction between two types of access to the target multiclass function: \emph{argmax (i.e. label) queries} %\amirg{argmax is a good name. Let's also use it in the intro}
and \emph{label-comparison queries}.

\begin{definition}[Supervision Oracles]
\label{def:oracles}
Given a target function $\fs:\X\rightarrow \R^{k}$, we define the following oracles:
\begin{align*}
     & A_{\argmax}^{\fs}(\xx) = \arg \max_{i\in[k]}\fs_i(\xx) \\
     & A_{\comparisons}^{\fs}(\xx,j_1,j_2) = \mathbf{1}[\fs_{j_{1}}(\xx)>\fs_{j_{2}}(\xx)]
\end{align*}
In the rest of the manuscript we simply use $A^{\fs}$ to denote the supervision oracle, where it's understood that if it receives an input $\xx$ it invokes the argmax oracle and if it receives a triplet $\xx, j_1, j_2$ it invokes the comparisons oracle.

%\amirg{Too many subscripts. One option is to overload $A$ such that if it receives $i,j$ it returns the comparison and otherwise the argmax. Let's keep both as macros for now.}

\end{definition}

\section{Passive learning}
\label{sec:passive}

%\paragraph{Passive learnability.}\amirg{There are currently too many definitions before we get to something interesting. I suggest having a section on PAC with both the definitions and the results, and then same for active.}

We define the sample and query complexities of PAC learnability using both argmax and label-comparisons supervision.%
\footnote{For simplicity, we consider a PAC notion where the goal is to return $\epsilon$-accurate solutions with constant probability (e.g. 14/15). 
\shay{Perhaps it is more natural and clean to consider the expected error.}
%Thus, the difficulty of learnability (the sample complexity for passive learning, or the query complexity for active learning) will depend on both the target accuracy parameter $\eps \in(0,1)$ and number of classes $k \in \N$. 
 }
 %Given functions $t_{1},t_{2}:(0,1)\times \N \to \N$, we say that $t_{1}$ dominates $t_{2}$ (denoted $t_{1}=O(t_{2})$ or $t_1 \leq t_2$) if $\lim_{\eps\to0,k\to\infty}\frac{m_{1}(\eps,k)}{m_{2}(\eps,k)}<\infty$; similarly, $t_{1}$ strictly dominates $t_{2}$ (denoted  $t_{1}=o(t_{2})$ or $t_1 < t_2$) if $\lim_{\eps\to0,k\to\infty}\frac{m_{1}(\eps,k)}{m_{2}(\eps,k)}=0$. Finally, we write $t_{1}=\Theta(t_{2})$ when $t_{1},t_{2}$ both dominate each other. We use $\tilde{O}, \tilde{\Omega}$ to hide factors that are logarithmic in $k$.
 %\gy{Not sure this is good, since I'm not taking into account the dependence on the dimension $d$. Fine for the 1d case but if we want to use this as a general definition probably not?}}
We begin with the usual \emph{passive learning} setup, and differentiate between the situation in which every example arrives with its argmax (i.e., the standard PAC setup), and where every example arrives with all the $k \choose 2$ pairwise label comparisons (essentially, the total order on the classes).

\begin{definition}[Sample complexity of passive learning with argmax supervision]
\label{def:passive_argmax}
Fix a distribution $\D$ over $\X$ and a target function $\fs:\X\to \R^{k}$. Let $\D_{\argmax}^{\fs}$ denote the distribution on $\X\times Y$ in which a sample $(\xx,y)\sim \D_{\argmax}^{\fs}$ is generated by drawing $\xx\sim \D$ and taking $y=A_{\argmax}^{\fs}(\xx)$.

We say that the sample complexity of passively learning a class $\set{\H^k}_{k \in \N}$ is $m_{\H}:(0,1) \times \N \to N$ if there exists a learning algorithm with the following property:
for every distribution $\D$ on $\X$, for every $k \in \N$, for every $\fs \in H^k$, and for every $\eps\in(0,1)$, given $m\geq m_{\H}(\eps, k)$ i.i.d samples from $\D_{\argmax}^{\fs}$, the algorithm returns an hypothesis $h$ s.t w.p at least $1-1/15$, $L_{\D}(h)\leq\eps$.%\amirg{say that this conincides with the stantdard notion of PAC learning.}
%\amirg{the order of quantifiers needs editing.  Should be for every $\epsilon$ if  $m>m_H$ then for all \ldots. It's not wrong as is, but seems nicer this way.}
%We refer to $m_\H \equiv m^{\argmax}_{\H,\passive}$ as the \emph{sample complexity of passively learning $\H$ with argmax supervision}.
\end{definition}

\begin{definition}[Sample complexity of passive learning with label-comparisons.]
\label{def:passive_comparisons}
Fix a distribution $\D$ over $\X$ and a target function $\fs:\X\to \R^{k}$. Let $\D_{\comparisons}^{\fs}$ denote a distribution on $\X\times\left\{ \pm1\right\} ^{k^{2}}$ where a sample $(\xx,\{b_{ij}\}_{i,j=1}^{k})$ is generated by drawing $\xx\sim \D$ and for $i,j \in [k]$, taking $b_{ij}=A_{\comparisons}^{f}(\xx;i,j)$.
\amirg{should be $f^*$ here?}
We say that the sample complexity of passively learning a class $\set{\H^k}_{k \in \N}$ is $m_{\H}:(0,1) \times \N \to N$ if there exists a learning algorithm with the following property:
for every distribution $\D$ on $\X$, for every $k \in \N$, for every $\fs\in H^k$, and for every $\eps\in(0,1)$, given $m\geq m_{\H}(\eps, k)$ i.i.d samples from $\D_{\comparisons}^{\fs}$, the algorithm returns an hypothesis $h$ s.t w.p at least $1-1/15$, $L_{\D}(h)\leq\eps$.

%We refer to $m_\H \equiv m^{\comparisons}_{\H,\passive}$ as the sample complexity of passively learning $\H$ with  label-comparison supervision.
\end{definition}

Note that in the latter setting, the learner receives strictly more information about every example than in the argmax supervision setting. Namely, the argmax can always be inferred from the total order on the classes. We will therefore consider label-comparisons as helpful in this setup if knowing all comparisons results in improvement 
to the sample complexity. Our first result is negative: in general, label-comparisons may not be helpful in the passive regime. 
\amirg{need to complement the below with an upper bound on the argmax setting. Otherwise it looks possible that this is a loose lower bound for the argmax and there is a gap.}\amirg{also we should probably say that the lower bound trivially applies to the argmax}
\begin{figure}[t]
    \centering
    \includegraphics[width=0.8\linewidth]{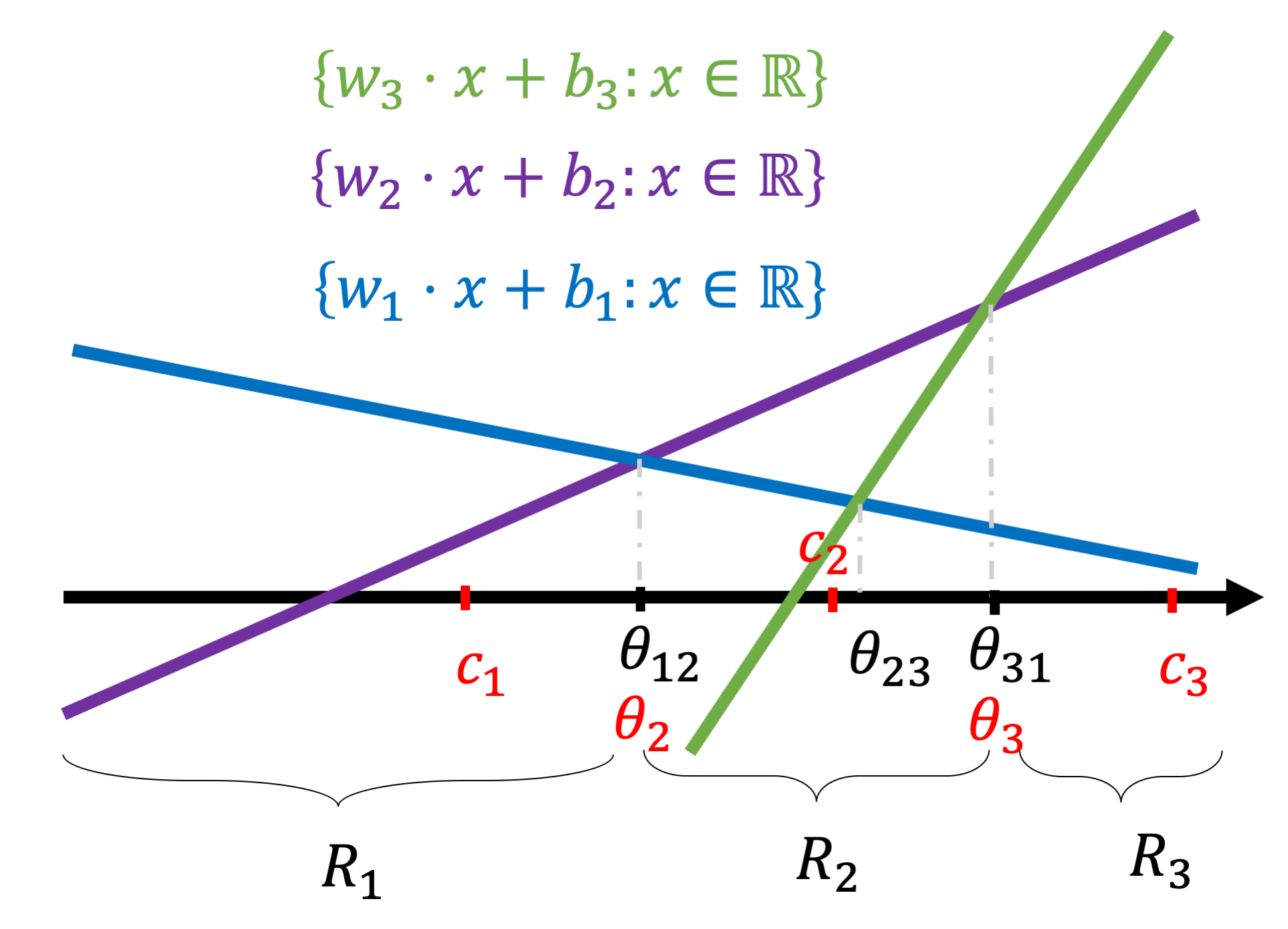} \vspace{-0.1in}
    \caption{Equivalent view of non-homogeneous linear classifiers in 1d in terms of 1NN classification. %Here, $R_i$ denote the decision-region of class $i \in [k]$. %\gale{Increase font of line equations. Also, the order R2, R1, R3 might confuse readers so better change to have the same order as indices}
    }
    \label{fig:1nn}
\end{figure}

\begin{theorem}
\label{lemma:passive}
Any algorithm that PAC learns $\H_{\linear}^{k,2}$ must use $m_{\H}(\epsilon,k)\in\Omega(k/\epsilon)$ samples, irrespective of whether it has access to argmax or label-comparison supervision.
\end{theorem}
%To prove Theorem \ref{lemma:passive}, we will use the fact that each $f \in \H^{k,2}_{\linear}$ can equivalently be parameterized via $k-1$ thresholds $\theta_2, \dots, \theta_k$. We also use  $\theta_{i,j}$ to denote the intersection of the lines $\set{w_i\cdot x + b_i : x \in \R}$ and $\set{w_j\cdot x + b_j : x \in \R}$ for two classes $i,j \in [k]$. Note that $\theta_i \equiv \theta_{i, i^-}$,  where $i^-$ the class preceding $i$ in the linear order of the classes in $\R$; See Figure \ref{fig:1nn} for an example with $k=3$ classes. 
\begin{proof}
For regular PAC learning (with argmax supervision), the standard approach for lower bounding the sample complexity is to lower bound the Natarajan dimension \citep{natarajan1989learning}. To extend this result to the setting of Definition \ref{def:passive_comparisons}, we employ a suitable variant of the dimension introduced in \cite{daniely2014optimal}. Following \cite{brukhim2022characterization}, we refer to it as the Daniely-Shwartz dimension. It provides a tighter lower bound on the sample complexity, and it is also easier to adapt to our label comparison setting.%, in which the learner has access to the full total order on the training examples, but is tested on correctly identifying the argmax. 

To emphasize the difference between functions mapping $\xx \in \X$ to a single class $y \in [k]$ and functions mapping $\xx \in \X$ to a total order over the $k$ classes, we will denote the former with $f$ and the latter with $f_\nabla$ (and likewise for hypotheses classes). We write $\arg\max f_\nabla(\xx) \in [k]$ for the class ranked first in the total order $f_\nabla(\xx)$.

\begin{definition}
\label{def:close}
Given a set $\{\xx_{1},\dots,\xx_{n}\} \subset \X$, we say that $f_{\nabla}$ and $g_{\nabla}$ are $\xx_{i}$-close if:
\begin{align*}
    \begin{cases}
f_{\nabla}(x_{j})=g_{\nabla}(x_{j}) & j\neq i\\
\arg\max f_{\nabla}(x_{j})\neq\arg\max g_{\nabla}(x_{j}) & j=i
\end{cases}
\end{align*}
\end{definition}

With this we can define a variant of Definition 12 in \cite{daniely2014optimal} for the case of extra supervision.

\begin{definition}[The Daniely-Shwartz dimension for label comparisons.]
\label{def:ds-dim}
A set $\{\xx_{1},\dots,\xx_{n}\}$ is shattered by $\H_{\nabla}$ if there exists a finite subset of functions $\H_{\nabla}'\subset \H_{\nabla}$ with the following property: for every $f_{\nabla}\in \H_{\nabla}'$ and for every $i\in[n]$, there exists $g_{\nabla}\in \H_{\nabla}'$ such that $f_{\nabla},g_{\nabla}$ are $\xx_{i}$-close. The Daniely-Shwartz dimension of $\H_{\nabla}$, $\text{dim}(\H_{\nabla})$, is the maximal cardinality of a shattered set.
\end{definition}

\gy{Changed:}
In Appendix \ref{supp:passive_lower_bound}, we prove that the sample complexity of passively learning a class $\H$ with label-comparisons (Definition \ref{def:passive_comparisons}) 
 is $\Omega(\text{dim}(\H)/\eps)$. 
Thus, 
our objective is to prove that  $\text{dim}(\H^{2,k}_{\linear}) \in \Omega(k)$. 
\shay{I think we should include an explicit derivation of the sample complexity lower bound. 
Also, we can give ourselves some credit for this: right now it sounds like an easy extension of the lower bound given by Amit and Shai, but I don't think it is the case.}

\shay{I think we should include an explicit derivation of the sample complexity lower bound. 
Also, we can give ourselves some credit for this: right now it sounds like an easy extension of the lower bound given by Amit and Shai, but I don't think it is the case.}\amirg{not sure if this is a new comment from Shay...}

To show this, we will construct a shattered set of size $k$ for $\H^{2,2k}_{\linear}$. Consider $2k$ labels of the form $(b,i)$, where $b\in \set{0,1}$ and $i\in \set{1,…,k}$. Partition the numbers $1,…, 3k$ to $k$ triples: $\set{1,2,3}$, $\set{4,5,6}$, $\dots$ $\set{3k-2, 3k-1, 3k}$. We claim that $k$ middle points, $S =\set{2, 5, \dots 3k-1}$ are shattered by $\H^{2,2k}_{\linear}$. Showing this requires defining a subset $\F$ of $\H^{2,2k}_{\linear}$ with the property of Definition \ref{def:ds-dim}. To define each function $f_{\nabla}\in \F$ we will use an equivalent parametrization of linear classifiers in 1d as 1NN classification. i.e., each total order in 
$\H^{2,2k}_{\linear}$ is parameterized by $\cc \in \R^{2k}$, where the total order $h(x; \cc)$ is the one implied by sorting the classes according to the distance of their centers $\cc$ to $x$. See Figure \ref{fig:1nn} for an illustration. With this parameterization in mind, $\F$ consists of all functions which satisfy the following: for each $i \leq k$, the centers corresponding to labels  $(0,i)$ and $(1,i)$ are located in the $i$’th triplet, and \emph{exactly one of them} is located in the middle of the triplet, on the point $3i-1$. By construction, $\card{\F} = 4^k$ (for each of the $k$ triplets we need to specify which of the two centers is located in the middle of the triplet,
and whether to locate the other center on the left or on the right of it).

To see that $S$ is shattered, consider $f_\nabla\in \F$ and a point $3i-1\in S$.
W.l.o.g, assume that the center located on $3i-1$ is $(0,i)$. We define $g_\nabla \in \F$ based on the location of the center of $(1, i)$, which by definition of $\F$, could be either to the right (on $3i$) or to the left (on $3i-2$). In the first case, $g_\nabla$ is obtained by shifting both centers one unit to the left:  in $g_\nabla$ the center $(0,i)$ is located on $3i-2$ and the center $(1,i)$ is located on $3i-1$. In the second case, $g_\nabla$ is obtained by shifting both centers one unit to the right: in $g_\nabla$ the center $(0,i)$ is located on $3i$ and the center $(1,i)$ is located on $3i-1$. By the definition of $\F$, $g_\nabla \in \F$. Crucially, $f_\nabla$ and $g_\nabla$ are $\set{3i-1}$-close (Definition \ref{def:close}): moving from $f_\nabla$ to $g_\nabla$ the center located on $3i-1$ (and therefore the argmax) has changed, but the total order for every other point in $S$ is remained unchanged, per the requirement of Definition \ref{def:close}.  This proves $S$ is shattered, and so $\text{dim}(\H^{2,k}_{\linear}) \in \Omega(k)$, as required.
\end{proof}

%\amirg{Need a sentence linking this to what we want to prove. e.g., a DS bound of $k$ then implies a lower bound of $k/\epsilon$. May need to say a word about why this still hold for our alternative notion of DS.}
%To prove Theorem \ref{lemma:passive}, we will use the fact that each $f \in \H^{k,2}_{\linear}$ can equivalently be parameterized via $k-1$ thresholds $\theta_2, \dots, \theta_k$. We also use  $\theta_{i,j}$ to denote the intersection of the lines $\set{w_i\cdot x + b_i : x \in \R}$ and $\set{w_j\cdot x + b_j : x \in \R}$ for two classes $i,j \in [k]$. Note that $\theta_i \equiv \theta_{i, i^-}$,  where $i^-$ the class preceding $i$ in the linear order of the classes in $\R$; See Figure \ref{fig:1nn} for an example with $k=3$ classes. 

% Gal: Added
\emph{Remark.} An interesting open question is whether this negative result can be extended to other classes (e.g. linear classifiers in higher dimensions). Technically, this requires lower bounding the the DS dimension of the class, as we did here for $\H_{\linear}^{2,k}$. We conjecture that for $d\gg 1$ the negative result can be extended in a distribution-specific manner (e.g., restricting to distributions with properties such as margin and sparsity); see the discussion in Appendix \ref{supp:passive_experiments}, where we report experimental results for the passive learning setting. 

% We do think that the negative (passive) result can be extended to linear classifiers in higher dimensions, although, as the reviewer notes, it may not be true for all data distributions, and could depend on distribution properties such as margin and sparsity. 

\section{Active Learning}
\label{sec:active}

Next, we consider the \emph{active} learning setting. Specifically, we focus on pool-based active learning, where the learner has access to unlabeled samples and can decide which queries to ask the oracle for (including not asking any queries). The performance of the algorithm is now measured in terms of the query complexity, namely the number of queries it makes to the labeling oracle in question.

\begin{definition}[Query complexity of active learning.]
\label{def:query_complexity}
The  query complexity of actively learning a class $\set{\H^k}_{k \in \N}$ is $q_{\H}:(0,1) \times \N \to N$ if there exists a function $m_{\H}:(0,1) \times \N \to N$ 
and a learning algorithm with the following property:
for every distribution $\D$ on $\X$, for every $k \in \N$, for every $\fs \in \H^k$, and for every $\eps\in(0,1)$, given $m\geq m_{\H}(\eps, k)$ i.i.d samples from $\D$ and at most $q_{\H}(\eps, k)$ queries to $A_{\argmax}^{\fs}$, 
the algorithm returns an hypothesis $h$ s.t w.p at least $7/8$, $L_{\D}(h)\leq\eps$. We refer to $q_\H$ as the query complexity of learning $\H$ with argmax supervision or with label-comparison supervision, depending \amirg{on?} the oracle $A^{\fs}$. \amirg{the unlabeled complexity is not addressed here. Is it part of the definition? Also we don't address it in the result we give for active sample complexity. Do we assume infinite unlabeled there?} \gy{to simplify, we focus on the query complexity  and in this definition only requires that there exists some unlabeled sample complexity function (I think of it as being polynomial in $1/\eps$, but I'm not sure we need to say this).}

\end{definition}

We note that every active learning algorithm that uses argmax queries can always be simulated using comparison queries: in the adaptive setting (where the choice of query to ask at time $t$ can depend on the previous answers), $k-1$ label-comparison queries suffice to implement a ``tournament'' that reveals the argmax. This provides a generic way to use the label-comparison oracle: simply request the label-comparison queries necessary for a ``regular'' active learner. We therefore say that \emph{comparisons are useful for active learning} if the number of label-comparison queries required to learn a class $\H$ is \emph{strictly lower} than the number of label-comparison queries required to simulate the best active learner that uses argmax queries to learn $\H$. 

Interestingly, the distinction between passive and active learning is important.
Our main result is that when the learner is allowed to decide which queries to request, label-comparisons are helpful: we provide a learning algorithm that uses label comparisons more efficiently than simply using them to implement the best ``regular'' active learner.

\begin{theorem}
\label{prop:active}
 The label-comparison query complexity for active learning $\H^{k,2}_{\linear}$ is $\tilde{O}(k \cdot \log \frac{1}{\eps})$, whereas the query complexity of simulating the best argmax active learner is $\tilde{\Omega}(k^2 \cdot \log \frac{1}{\eps})$.
\end{theorem}

The proof of Theorem \ref{prop:active} will employ a specific multiclass to binary reduction that uses the concept of the \emph{label neighborhood graph} of the target classifier. Intuitively, two classes $i$ and $j$ are considered neighboring if they share a decision boundary; i.e., there are two arbitrarily close points in $\R^d$, where for one the argmax is $i$ and for the other the argmax is $j$. See Figure \ref{fig:example_graph} for an example of the label neighborhood graph of a linear classifier in $d=2$.

\begin{figure}
    \centering
    \includegraphics[width=0.95\linewidth]{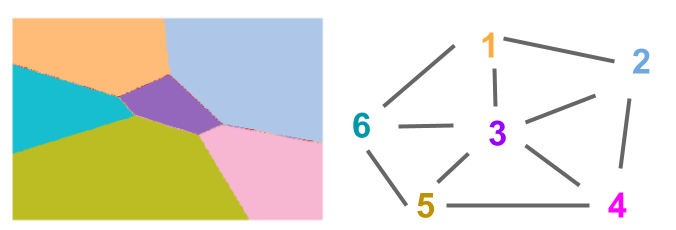}
    \caption{Decision regions of a linear classifier in 2d (left) and its corresponding label neighborhood graph (right).}
    \label{fig:example_graph}
\end{figure}

\begin{definition}[Label Neighborhood graph]
\label{def:neighborhood}
Fix a continuous function $f: \R^d \to \R^k$. The neighborhood graph $G = G(f)$ is an undirected graph on $k$ vertices, with an edge between vertices $i \in [k]$ and $j \in [k]$ if and only if there exists $\xx \in \R^d$ for which for every $r \in [k]$, $f_i(\xx) = f_j(\xx) \geq f_r(\xx)$.
%\shay{This definition seems to implicitly assume that the $f$ is continuous, which is fine and natural, but perhaps worth a comment.} 
\end{definition}

To simplify notation, we use $(i,j)\in G$ to refer to an edge in $G$, and $\deg(G)$ for the total number of edges. The degree of 
 $i \in [k]$ is the number of neighbors $i$ has in $G$. 

\begin{algorithm}[t]
  \caption{\textbf{$\mathtt{NbrGraphM2B}$: active learning of $\H^{k,d}_{\linear}$ using $G$.}}
  \label{algo:beyond-1d}
      \begin{algorithmic}
        \STATE {\bfseries Input:}  $\eps > 0$, a binary active learning algorithm $\mathtt{B}$ with query complexity $q_b(\gamma)$, a  neighborhood graph $G$.
        \STATE {\bfseries Output:} $f: \X \to \R^k$.
        \vspace{1mm}
        \FOR{$(i,j) \in G$}
        \STATE %Use $q_b(\eps/ \deg(G))$ label-comparison queries to learn a binary classifier that distinguishes class $i$ from class $j$ with error at most $\eps/\deg(G)$.
        Use  $\mathtt{B}$ to learn a binary classifier that distinguishes class $i$ from class $j$ with error at most $\eps / \deg(G)$.
        \ENDFOR
        \STATE Let $C$ denote the set of all the learned binary classifiers.
        \STATE Return $f^{(G, C)}$ (Definition \ref{def:f_GC}).
        \vspace{1mm}
      \end{algorithmic}
\end{algorithm}

We next define $\ALalgo$ (Neighborhood Graph Multiclass-to-Binary), a procedure for actively learning a multiclass classifier $f$ using a neighborhood graph $G$ (see Algorithm \ref{algo:beyond-1d}). Given as input a binary active learning algorithm %\amirg{maybe use clearer term for the "regular" active} 
and a neighborhood graph, it uses comparison queries to learn a binary classifier for distinguishing every pair of neighboring classes $i,j$ in $G$. It then aggregates these into a multi-class classifier using the following scheme:
%\gale{In various places in algorithm and definition we have $(i,j) \in G$ where we actually mean $(i,j) \in e(G)$. We should either clarify this from the get go, or use the more annoying $e(G)$ notation.}\gy{resolved.}

\begin{definition}[Binary to multiclass aggregation.]
\label{def:f_GC}
Fix $(G,C)$, where $G$ is a neighborhood graph and $C=\{h_{ij}\}_{(i,j)\in G,i<j}$ is a collection of binary classifiers, one for every edge in $G$. The graph-based aggregation of $(G,C)$ is a function $f^{(G,C)}:\X\to\R^{k}$ defined as follows:
%
%\begin{equation*}
\[
f_i^{(G,C)}(\xx) = \frac{\sum_{(i,j)\in G}\boldsymbol{1}[h_{ij}(\xx)\geq0]}{\sum_{(i,j)\in G}\boldsymbol{1}}
\]
%\end{equation*}
%
\amirg{This doesn't predict a label. It output probabilities. Do we really want it this way?}\gy{I think it's okay, we defined earlier that accuracy is measured by the argmax.}
Namely, the label of $\xx$ is the class in $[k]$ that won the largest fraction of ``duels'' against its neighbors in the graph $G$.
\end{definition}

An important component in analyzing
$\ALalgo$ is the following lemma. It establishes that when invoked w.r.t the true neighborhood graph $\GS$, if the binary classifiers are sufficiently accurate, then so is the resulting  multiclass classifier. See Supplementary \ref{sec:aggregation_proof} for proof.
\begin{lemma}
\label{lemma:aggregation}
Fix a distribution $\D$ on $\X$ and a classifier $\WWS$. Fix $(G,C)$. If $G = G(\WWS)$ and every $h_{ij}\in C$ has error at most $\eps/\deg(G)$, then $f^{(G,C)}$ has error at most $\eps$. \amirg{have we defined $e(G)$?} \gy{yes, it was defined after Def 5.3.}
\end{lemma}

From this, we obtain the following upper bound on the query complexity of learning $\H_{\linear}^{k,d}$ using label comparisons.%, in terms of the query complexity of binary active learning .

\begin{corollary}
\label{cor:algoALcomplexity}
If the target neighborhood graph $\GS$ is known,\amirg{don't need to assume it's known. We can just give the result.}\gy{I think we should leave this as is. It makes sense to explicitly say this bc this is the query complexity of an algorithm that explicitly uses $\GS$.} the label-comparison query complexity of learning $\H_{\linear}^{k,d}$ is $O(\deg(\GS) \cdot q_b(\eps/\deg(\GS))$, where $q_b(\gamma)$ is the query complexity of active learning in the binary case (i.e. $k=2$).
\end{corollary}

Corollary \ref{cor:algoALcomplexity} suggests that label-comparisons will be useful when (i) the target neighborhood graph is sparse (has low degree), and (ii) it can be learned with relatively few label-comparisons.\amirg{not clear what the second point is. Does it refer to the $q_b$?}\gy{it refers to learning $\GS$} We are now ready to prove Theorem \ref{prop:active}: we will show that for learning $\H_{\linear}^{k,2}$ (the class for which we demonstrated comparisons are not useful in the passive setting), both these conditions hold. Hence, comparisons indeed provide a provable gain over argmax supervision.

\emph{Proof of Theorem \ref{prop:active}}: We will begin by instantiating the bound from
Corollary \ref{cor:algoALcomplexity} for $d=1$. Consider the degree of the neighborhood graph. Using the equivalent parameterization of linear classifiers in $d=1$ (see Figure \ref{fig:1nn}), it follows that every class $i \in [k]$ has at most $2$ neighbors: exactly the preceding and succeeding classes in the sorted order of the classes. Thus, for every $\fs \in \H_{\linear}^{k,2}$, $\deg(G(\fs)) = O(k)$. Second, active learning in $d=1$ is well-understood: unlike higher dimensions, the distribution-free query complexity of active learning for two classes is $q_b(\gamma)=\log (\frac{1}{\gamma})$ using binary search over $\R$ \citep{dasgupta2004analysis}. Plugging both of these facts into the upper bound of Corollary \ref{cor:algoALcomplexity}, we conclude that the query complexity for learning  $\H_{\linear}^{k,2}$ when the target neighborhod graph is known is $O(k \cdot \log \frac{k}{\eps})$. 

\begin{algorithm}[t]
  \caption{\textbf{Learning $G(\fs)$ for $\fs \in \H^{k,2}_{\linear}$. }}
  \label{algo:learning_G_1d}
      \begin{algorithmic}
        \STATE {\bfseries Input:} $n$ i.i.d samples from $\D$, $x_1, \dots, x_n$.
        \STATE {\bfseries Output:} A neighborhood graph $G$.
        \vspace{1mm}
        \STATE Set $x_L = \min_ix_i$ and $x_R =\max_ix_i $.
        
        \STATE Use a comparison sorting procedure to obtain a total order over the $k$ classes, $i_1 \succ \dots \succ  i_k$. Every time the sorting procedure requires the comparison between  classes $i,j \in [k]$, determine that  $i$ appears before $j$ if and only if (i)  $A^f_{\comparisons}(x_L, i,j) = 1$ and   $A^f_{\comparisons}(x_R, i,j) = 0$, or (ii) $A^f_{\comparisons}(x_L, i,j) = 1$ and $A^f_{\comparisons}(x_R, i,j) = 1$.
        
        %\FOR{$i < j \in [k]$}
        %\STATE Determine the partial order between $i$ and $j$: %\shay{This is imprecise: it can be that $i$ appear after (to the right) of $j$ but item (ii) is satisfied. (Imagine that $x_L=x_R$ and $i$ ($j$) is to the right (left) of it and $i$ is closer.)}
        %\ENDFOR

        \STATE Define a neighborhood graph $G$ with an edge between $i$ and $j$ iff classes are consecutive in the learned total order. 
        \STATE Return $G$.
        \vspace{1mm}
      \end{algorithmic}
\end{algorithm}

%\gy{Modified (start)}

Next, we turn to the question of learning $\GS$ using label-comparison queries. Towards this, consider  Algorithm \ref{algo:learning_G_1d}. The algorithm receives a sample of $n=O(1/\eps)$ points from $\D$ and uses exactly $2k\log k$ label comparisons to return a neighborhood graph $G$. As we claim below the graph $G$ will be identical to $\GS$, except for possibly a set of edges pertaining to classes outside $S$ whose overall probability under $\D$ is smaller than $\eps$.
The key observation behind Algorithm \ref{algo:learning_G_1d} is that we can use exactly two label-comparison queries to infer whether a class $i$ appears before a class $j$, as long as both classes are ``represented'' in $S$.\footnote{We say a class $i$ is represented in $X$ if the position of $i$ in the total order of all the classes is greater-equal than the position of $\min(S)$ and smaller-equal than the position of $\max(S)$.} We can therefore use a total of $2k\log k$ queries to infer the total order of all the ``represented'' classes.

It remains to argue why $O(1/
\eps)$ samples suffice to guarantee that with high probability, classes that are not ``represented'' by $S$ have mass at most $\eps$.
To see this, fix $\D$ on $\R$ and denote $F(z) = \Pr_{x\sim \D}[x<z]$. Let $z$ be such that $F(z)=\eps$. We are interested in the number of samples $n$ required to guarantee that $\Pr_{S \sim \D^n}[\min(S) > z] < \delta$. Now,
\begin{equation*}
    \Pr_S[\min(S) > z] = \br{(1-F(z)}^n = (1-\eps)^n \leq \exp(-n\cdot \eps) 
\end{equation*}

And $\exp(-n\cdot \eps)  \leq \delta \iff n \geq \frac{1}{\eps}\log \frac{1}{\delta}$. Similarly, the same number of samples can be used to bound the the "tail" beyond $\max(S)$. Union-bounding over both events yields the required result.

To summarize, the full procedure for actively learning $\H^{k,2}_{\linear}$ is to run  
$\ALalgo$ with the neighborhood graph $G$ that is returned by Algorithm \ref{algo:learning_G_1d}. Combining Lemma \ref{lemma:aggregation} and the analysis of Algorithm \ref{algo:learning_G_1d}, we conclude that this procedure has an overall unlabeled sample complexity of $O(1/\eps)$, and an  overall query complexity of $O(k \log k + k \log \frac{k}{\eps}) = \tilde{O}(k \cdot \log \frac{1}{\eps})$.

To conclude the proof of Theorem \ref{prop:active}, it remains to lower bound the complexity of learning with argmax queries. We will prove that $\Omega(\frac{k}{\log k}\log\frac{k}{\eps})$ argmax queries are needed. This will imply that simulating any argmax active learning requires at least $\tilde{\Omega}(k^2 \cdot \log \frac{1}{\eps})$ label-comparisons. We will prove this via the label revealing task \citep[e.g., see][]{kane2017active}, where the goal is to reveal the correct labels of a given (realizable) sample of $n$ points, and show that  $O(\frac{k}{\log k}\log n)$ argmax queries are needed to reveal all $n$ labels.
%\footnote{Note that since we are concerned with the realizable setting, this indeed implies the lower bound on the query complexity.}.\gale{Not sure footnote is helpful beyond the realizable emphasis above until we point to a ref here}

Towards this, fix $n$  points and consider a tree that denotes the run of an active learning algorithm (with nodes being the queries asked and the children the possible answers). Note that the number of unique labelings corresponds to the number of leaves in the tree and the query complexity corresponds to the depth of the tree, which we denote $q$. The number of ways to arrange $n$ points into $k$ classes in 1d is $k!{n \choose k-1}$ ($k!$ options for ordering the classes and then ${n \choose k-1}$ options for locating the thresholds). Since the degree of the tree is $k$ for argmax queries, it must be that the 
$k^q \geq k!{n \choose k-1}$, which implies\footnote{Using the fact that $\log\big(k!{n \choose k-1}\big) = \log k!+\log{n \choose k-1} = k\log k+\log\left(\left[\frac{n}{k}\right]^{k}\right) = k\log k+k(\log n-\log k) = k\log n$} a lower bound $q \geq O(\frac{k}{\log k}\cdot\log n)$.

Together, this concludes the proof of Theorem \ref{prop:active}. $\square$

%i.e., if $m_{\H,\activelearning}^{\comparisons}=O(m_{\H,\activelearning}^{\argmax})$ and $q_{\H,\activelearning}^{\comparisons}=o(k\cdot q_{\H,\activelearning}^{\argmax})$.

%\paragraph{Beyond $d=1$.} 
Our analysis suggests that when we can efficiently learn $\GS$ and it is sparse, label-comparisons provide a gain over argmax queries. We showed this  when $d=1$, and it is natural to ask to what happens for $d>1$. This requires addressing both the question of what is the binary active learning primitive that we use, as well as the questions of sparsity and learning the graph. See Supplementary \ref{supp:d_gt1} for discussion of these aspects.

\section{A general purpose active learning algorithm}
\label{sec:general}

\begin{algorithm}[t]
  \caption{\textbf{$\ALGDalgo$ }}
      \begin{algorithmic}
        \STATE {\bfseries Input:}  Label neighborhood graph $G$, buffer size $R$, steps $T$, confidence parameter $\tau$, learning rate $\eta$, comparison oracle $A^{\fs}$.
        
        \STATE {\bfseries Output:} classifier $h(\cdot; \WW)$, number of comparisons $q$.
        \vspace{1mm}
        \STATE Initialize $\WW^{(0)}$ , $L=0$, $q=0$, $b=0$.
        \FOR{$t = 1, 2, \dots, T$}
        \STATE Sample $\xx \sim \D$.
        \STATE Sample $(i,j)$ uniformly from the edges of $G$.
        
        \IF{ $\card{h_i(\xx;  \WW^{(t-1)}) - h_j(\xx;  \WW^{(t-1)})} < \tau$}  
        \STATE Obtain oracle comparison $c =2(A^{f*}(\xx,i,j)-0.5)$
        \STATE $L \pluseq \log(1+e^{-c\left(h_i(\xx;\WW)-h_j(\xx;\WW)\right)})$. 
        \STATE $q \pluseq 1$, $b \pluseq 1$.
        \ENDIF

        \IF{ $b \geq r$}
        \STATE Update $\WW^{(t)} \leftarrow \WW^{(t-1)} - \eta\cdot \frac{\partial L}{\partial \WW}$
        \STATE Clear buffer: $L = 0$, $b=0$. 
        \ENDIF
        \ENDFOR
        \vspace{1mm}
      \end{algorithmic}
      \label{fig:oursgdalg}
\end{algorithm}
%\ell_{BCE}(\xx, i,j, \WW^{(t-1)})$
%\mbox{sign}\left(h_i(\xx;  \WW^*) - h_j(\xx;  \WW^* \right)$.
The approach of Algorithm \ref{algo:beyond-1d} is to explicitly learn $\deg(G)$ binary classifiers and aggregate them into a single classifier, that is not in $\H_{\linear}^{k,d}$. For simplicity of optimization, we will prefer to work with models in $\H_{\linear}^{k,d}$. To do so, in Algorithm \ref{fig:oursgdalg}  we present the $\ALGDalgo$ algorithm, a natural variation which can work directly with such models.
%We describe $\ALGDalgo$ in 

It works as follows: we first initialize a multiclass model $h(\cdot ; \WW): \R^d \to \R^k$ (e.g. $\WW \in \R^{k,d}$ for a linear model, but $h$ can also be a neural network). For every data point $\xx$, we sample an edge $(i,j)$ in the graph $G$. This edge is a candidate label comparison. To decide whether to query it or not, we evaluate the difference in logits between labels $i$ and $j$. If this difference is smaller than $\tau$ we query the pair $(i,j)$  and add a binary cross entropy term that encourages the logit difference to have the correct sign. Once we accumulate sufficiently many comparisons, we perform an update step.

The remaining practical question is which graph $G$ to use. Recall that $\GS$ has an edge $(i,j)$ iff $\exists\xx$ where $j$ was the 2nd best label and $i$ was the argmax.  For 1d,\amirg{1D or $d=1$} we showed this could be learned from data effectively. We leave the general case open, and consider here practical recipes for $G$. The simplest approach is to base $G$ on prior knowledge regarding which classes are expected to be neighbors (e.g., via distances on their word embeddings, or other co-occurrence statistics). Another practical case is when first and second best labels are available without the $\xx$ values (e.g., consider asking individuals what are their first and second most favorite products, without keeping user info). Note that in this case, we will receive evidence of edges only for $\xx$ values sampled from $\D$. This corresponds to an empirical notion of the neighborhood graph, which we define below.  

%A second important practical issue is regarding the neighborhood graph of the target function, $\GS$. While we have shown mathematically that comparisons are helpful when this graph is sparse, the fact that it draws an edge between classes $i$ and $j$ if \emph{there exists some $\xx \in \R^d$} for which $i,j$ are the top two classes makes it challenging to work with. Technically, even if the target concept is linear and known, computing $\GS$ requires solving ${k \choose 2}$ linear programs. Of course, in practice the target concept is unknown, and this definition suggests it may be hard to learn in the pool-based active learning model. Conceptually, we don't anticipate annotators being able to respond to label-comparisons outside the manifold of natural examples. Thus, we will hereby work with an \emph{empirical} version of the neighborhood graph. We state it w.r.t the underlying distribution $\D$ on $\X$, so it can be readily approximated using samples from $\D$.
\begin{definition}[Empirical Label Neighborhood graph]
\label{def:emp-neighborhood}
For a target function $\fs: \R^d \to \R^k$, the neighborhood graph $G_\D(f)$ is an undirected graph on $k$ vertices, where there is an edge between vertices $i$ and $j$ if and only if there exists $\xx \in \R^d$ whose probability under $\D$ is non-zero, and for which for every $r \in [k]$, $\fs(\xx)_i = \fs(\xx)_j \geq \fs(\xx)_r$.
\end{definition}
By definition, $\GS_\D\subseteq\GS$. One might hope that the discarded edges will not impact accuracy w.r.t $\D$. However, in the worst-case this is not true. Specifically, in proving Lemma \ref{lemma:aggregation} we used the fact that when $B$ is given by the \emph{true} binary classifiers (i.e. $h_{ij}=\WWS_{i}-\WWS_{j}$), the aggregated classifier $f^{(\GS,C)}$ has perfect accuracy on $\D$. This may fail for $\GS_\D$:   $f^{(\GS_\D,C)}$ may err on examples supported in $\D$; See Supplementary Figure \ref{fig:G_D_example} for an example. In Section \ref{sec:experiments} we observe that the performance of both graphs is comparable.
%We therefore suggest using $\GS_\D$ as a proxy for $\GS$ as a heuristic, which we validate in our experiments. Specifically, in Section \ref{sec:experiments} we show that empirically  $\GS_\D$ serves as a good proxy to $\GS$ on both synthetic and real data.

\section{Experiments}
\label{sec:experiments}

In this section we evaluate our label-comparisons algorithm $\ALGDalgo$ on synthetic as well as real data. %We describe the experimental setup in Section %\ref{sec:experiments:implementation}, and present %the results for synthetic and real data in %Sections \ref{sec:experiments:synthetic} and %\ref{sec:experiments:real}.

%Our theoretical results suggest that when the neighborhood graph of the target classifier is sparse, using the neighborhood graph to ``guide'' which label-comparisons we request from the oracle will be more efficient than simply using them to simulate a standard active learning algorithm. In this section we validate this on both synthetic and real-world data. 

%In Section \ref{sec:experiments:implementation} we describe our setup, including theactive learningprocedure we employ and the benchmarks we consider. We then report our experiments on synthetic and real data in Sections \ref{sec:experiments:synthetic} and \ref{sec:experiments:real}.

%\subsection{Setup and implementation}
%\label{sec:experiments:implementation}

We consider the online active learning scenario. At each
 round $t \in [T] = \set{1,...,T}$, the learner receives a batch of points drawn i.i.d. according to $\D$ and must decide which queries to request from the oracle $A^{\fs}$ (including not requesting any queries). 
 %All the methods we consider need to decide whether or not to query a given pair of labels $i,j$ (the difference between the methods is which $i,j$ to consider). We do this by simply thresholding the difference of logits $|W_i\xx - W_j\xx|$. If it is above a certain fixed threshold, we query the oracle on the pair $i,j$. This is essentially an implementation of the \emph{Best-versus-Second Best} (BvSB) confidence-based criteria, proposed in the context of multiclass learning \cite{joshi2009multi}.
 %It requests the label of an instance if the current model is ``uncertain'' in the prediction on this instance, in that the multi-class margin (the gap between the logits of the first and second best classes) is small. 
%\paragraph{Baselines.} 
We compare the following methods:
\begin{itemize}[leftmargin=*]
    \item  $\ALGDalgo (G)$: This is our algorithm which takes as input a graph $G$ and, for a given $\xx$, only considers label pairs in $G$ as possible pairs to query. For the given $\xx$, we iterate over all $(i,j) \in G$.  For each pair we check if $|W_i\xx - W_j\xx|$ is smaller than a fixed threshold. If it is, we query this pair. We consider different versions of $\ALGDalgo(G)$, that use different graphs.
    \item $\PassiveTournament$: This baseline uses label comparisons to simulate a standard argmax-based active learning algorithm \citep{joshi2009multi}. Namely, for each $\xx$, we evaluate the logits $W_y\xx$ and  query $\xx$ if the difference between the first and second best logits is below some threshold. In the standard argmax setting, we would have requested the label of $\xx$. With label 
    comparisons, we need to do this using $k-1$ active comparisons. Namely, we perform a tournament between labels to reveal the maximizer.
    %the BvSB AL. I.e., whenever the BvSB criterion asks for the argmax of $\xx$, we run a ``tournament'' to reveal the argmax (label) of $\xx$ using $k-1$ comparisons (dueling between two classes and discarding the loser, until all $k$ classes have been iterated over). 
    \item $\ActiveTournament$: It may seem wasteful to ask for $k-1$ comparisons as above, since we may be sufficiently confident in some of these comparisons. We thus consider an ``active tournament'' algorithm: whenever the current model is sufficiently confident in a given pair in the tournament, we take the model's answer, and do not query for it.
    %Similar to $\PassiveTournament$, but whenever the current model $\WW^t$ is sufficiently confident in the result of a duel between two classes, we use the prediction of $\WW^t$ as the result of the duel (without ``charging'' for a comparison, since the oracle was not queried). Note that $\ActiveTournament$ uses a more refined notion of uncertainty, and thus might be better when the number of classes is large.\footnote{E.g.  suppose that on input $\xx$, $\WW^t$ assigns  equal probability to two classes: the true argmax $y$, and another label $y'$. To the other classes it assigns probability zero. Here $\PassiveTournament$ will ``waste'' $k-2$ queries by comparing the true label to all the other labels, whereas $\ActiveTournament$ will only ask for the comparison $(\xx, y, y')$.}
\end{itemize}
 %One criticism is that the multi-class notion of uncertainty, BvSB, is too coarse: perhaps the classifier is only ``unsure'' about the classes ranked 1 and 2 -- but our approach charges it for $k-1$ comparisons regardless. We therefore include an additional baseline which we refer to as an \emph{active tournament}: in the process of dueling classes to determine the argmax of $\xx$, whenever the current model is sufficiently certain about the result of the duel between $i$ and $j$, we take its prediction as the result of the duel, and continue the tournament (and do not charge for a comparison as a result). Hence if there's a significant gain from using a finer grained notion of uncertainty alone, we expect the active tournament to reflect this. 

\paragraph{Evaluation.} In online active learning, the quality of an algorithm is measured by its accuracy after $T$ rounds, and the total number of comparisons requested within these $T$ rounds. We use a linear teacher model to simulate the comparison oracle (Definition \ref{def:oracles}), and measure accuracy as the categorical accuracy\footnote{Specifically, use Top-K accuracy, where $K=0.1 \cdot k$.} %That is, an example is considered correctly labeled if the true label is in the top 10\% of predicted classes.
on the test set, w.r.t the teacher's argmax. We use an \emph{accumulating  buffer} mechanism to control for the number of parameter updates across methods (each method accumulates the requested comparisons until the buffer is full, and only then performs a gradient update). 
%\gale{I didn't get why parameter update schedule is part of the evaluation procedure}\amirg{it's mostly to explain why the comparison is "fair"}

%One important issue is that the \emph{adaptivity} of active learning creates challenges for properly comparing different methods. For example, taking a gradient step at the end of every batch (i.e., over the queries each method selected) will unfairly favour methods that are more aggressive in terms of the degree to which they filter the examples. I.e., if a method consistently ``generates'' less comparisons from a given batch of examples, it will perform more model updates per fixed comparison budget. Thus, it's crucial to control for the number of updates across all methods; towards this objective, we implement an \emph{accumulating  buffer} mechanism.  Each method stores the comparisons it requested, aggregating them across batches, until the buffer reaches full capacity -- only then do we do a gradient update across the comparisons popped from the buffer.

%We summarize our algorithm and the two baselines in Figure 3. \gy{Add algorithm float, probably in an appendix?}.

\subsection{Synthetic data}
\label{sec:experiments:synthetic}
In this section we validate our theoretical findings from Sections \ref{sec:passive} and \ref{sec:active} on synthetic data. Specifically, for $d \in \N$ and $\hat{k} \in \N$ we consider $\D$ to be the uniform distribution on a unit sphere in $\R^d$ and draw a random linear target model $\WWS \in \R^{\hat{k},d}$. This yields a multiclass classifier with $k \leq \hat{k}$ distinct decision regions (``effective classes''). We draw data from $\D$ and divide it into training and test sets. %We use the training set to train the methods and report accuracy on the test set.
%\newline

{\bf Sparsity of the neighborhood graph. }We begin by computing the sparsity level of both the true neighborhood graph $\GS = G(\WWS)$ and the empirical neighborhood graph  $\GS_\D = G_\D(\WWS)$, where the latter is computed w.r.t the training set. We define the \emph{sparsity level} as the number of edges in $G$, divided by $k \choose 2$ (i.e,. the number of edges in a complete graph). \gale{In all figures, width of lines is out-of-proportion to fonts and fig size. Make thinner} In Figure \ref{fig:sparsity} we plot the sparsity level as a function of $k$ and $d$, as averaged over 25 random target models. We see that the empirical sparsity level tracks the true sparsity level, and that for a fixed dimension $d$, both decrease with the number of effective classes $k$. This confirms that we expect the sparsity to ``kick in'' when $k \gg d$. 
\begin{figure}[t]
    \centering
    \includegraphics[width=0.80\linewidth]{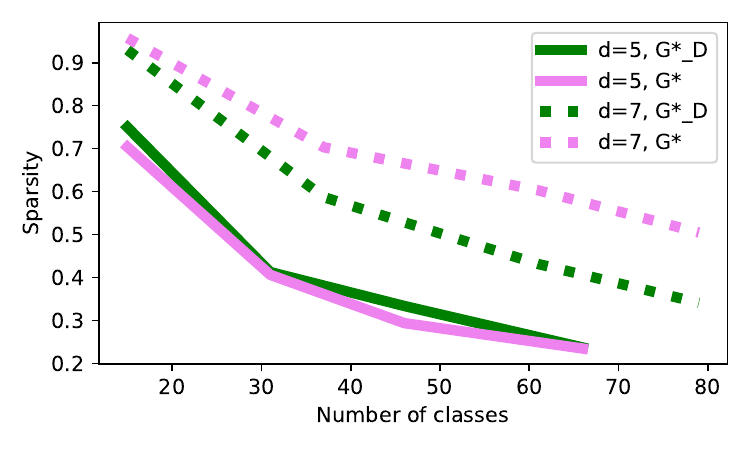}
    \caption{Sparsity level for a random linear model as a function of the number of effective classes $k$ for $d=5, 7$}
    \label{fig:sparsity}
\end{figure}
%\begin{figure}
%    \centering
%    \includegraphics[width=1.0\linewidth]{active_passive.png}
%    \caption{Learning $h(\cdot ;\WWS)$ for $k=50$ and $d=5$ when the examples to query are chosen actively (left) vs passively (right).}
%    \label{fig:active_passive}
%\end{figure}
%\paragraph{Active vs passive.} Next, we demonstrate that actively selecting which examples to query plays an important role in the efficiency gain introduced by using $\ALGDalgo$. Figure \ref{fig:active_passive} compares the performance of $\ALGDalgo$ w.r.t $\GS_\D$ %(as applied w.r.t either the true graph or the empirical graph\amirg{I see just one graph in the figure}) 
%with the performance of the passive tournament baseline, under both active learning (examples are selected based on the active learning criterion) and passive learning (examples are selected uniformly at random\amirg{not clear if this is what we referred to as passive, where all pairwise are available}). In line with our findings in Section \ref{sec:passive}, we observe that in the low error regime, comparisons are mostly useful in the active setting. 
%\amirg{consider dropping this part}
%\amiw{this paragraph is hard to follow and 6 should come after 5}
%\newline

{\bf Comparisons of Active Learning Methods.} We next compare the different baselines described above. For $\ALGDalgo(G)$ we consider multiple variations, that use different versions of the graph $G$.
In Figure \ref{fig:synthetic_results} we report the performance of $\ALGDalgo$ relative to several natural baselines. %\gale{We need to be clear about what we call baselins. At the start of the section we say we compare 3 things (2 baselines) and now there are 2 variants of NbrGraphSGD compared to 3 baselines one of which is also a variant. The reader may be confused. Maybe mention the variants as early as the 'Baselines' paragraph?}.
First, it can be seen that the active tournament outperforms the passive one, suggesting that indeed some tournament queries can be avoided. Yet $\ALGDalgo$ outperforms 
the tournament baselines, indicating that tournament comparisons are generally not the optimal approach.
Within the $\ALGDalgo$ methods, using the true graph (either $G^*$ or $G^*_D$) provides the best performance, indicating that the graph plays an important role in active learning efficacy, and that $\ALGDalgo$ can use this structure.
%to a \emph{complete graph} provides a gain over both methods. Using the actual neighborhood graph (either $\GS_\D$ or $\GS$) is significantly better than all baselines. \amiw{whats dependent/independent?}

\begin{figure}[t]
    \centering
    \includegraphics[width=0.75\linewidth]{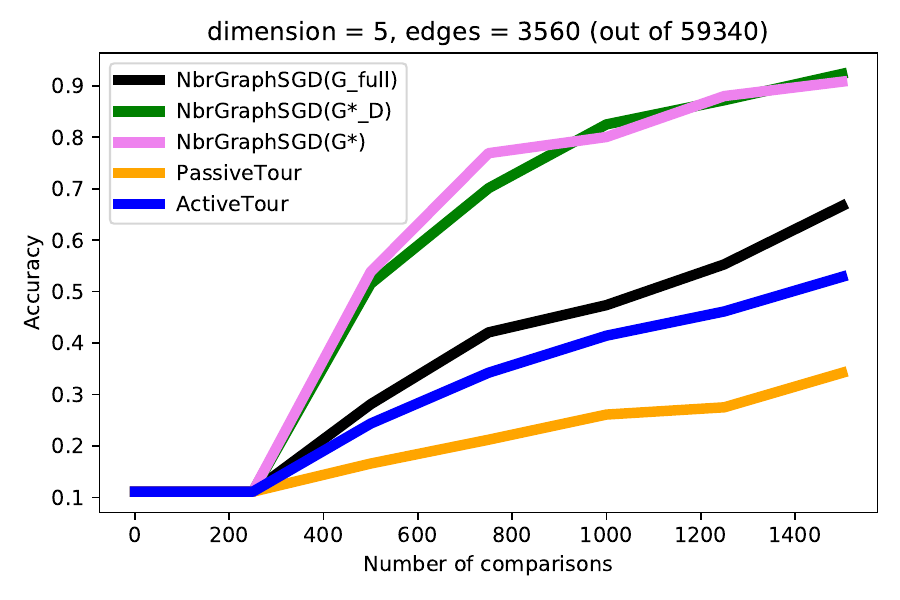}
    \caption{Comparing algorithm $\ALGDalgo$ w.r.t $\GS_\D$ (green) and $\GS$ (purple) against three baselines: passive tournament (yellow), active tournament (blue) and algorithm $\ALGDalgo$ with respect to a complete graph (black).}
    \label{fig:synthetic_results}
\end{figure}

%\gy{Explain synthetic data generation and experiment setup. Report params. Show: (1) active and passive gap, (2) comparing all five methods on $d=1$ or $d=2$ (where we have theory), indeed we see the random graph is very bad, as we anticipate (3) comparing all five methods on $d > 1$, indeed we see that when $k \gg d$ and the graph is sparse our approach is much better than baselines.}

\begin{figure}[t]
    \centering
    \includegraphics[width=1.0\linewidth]{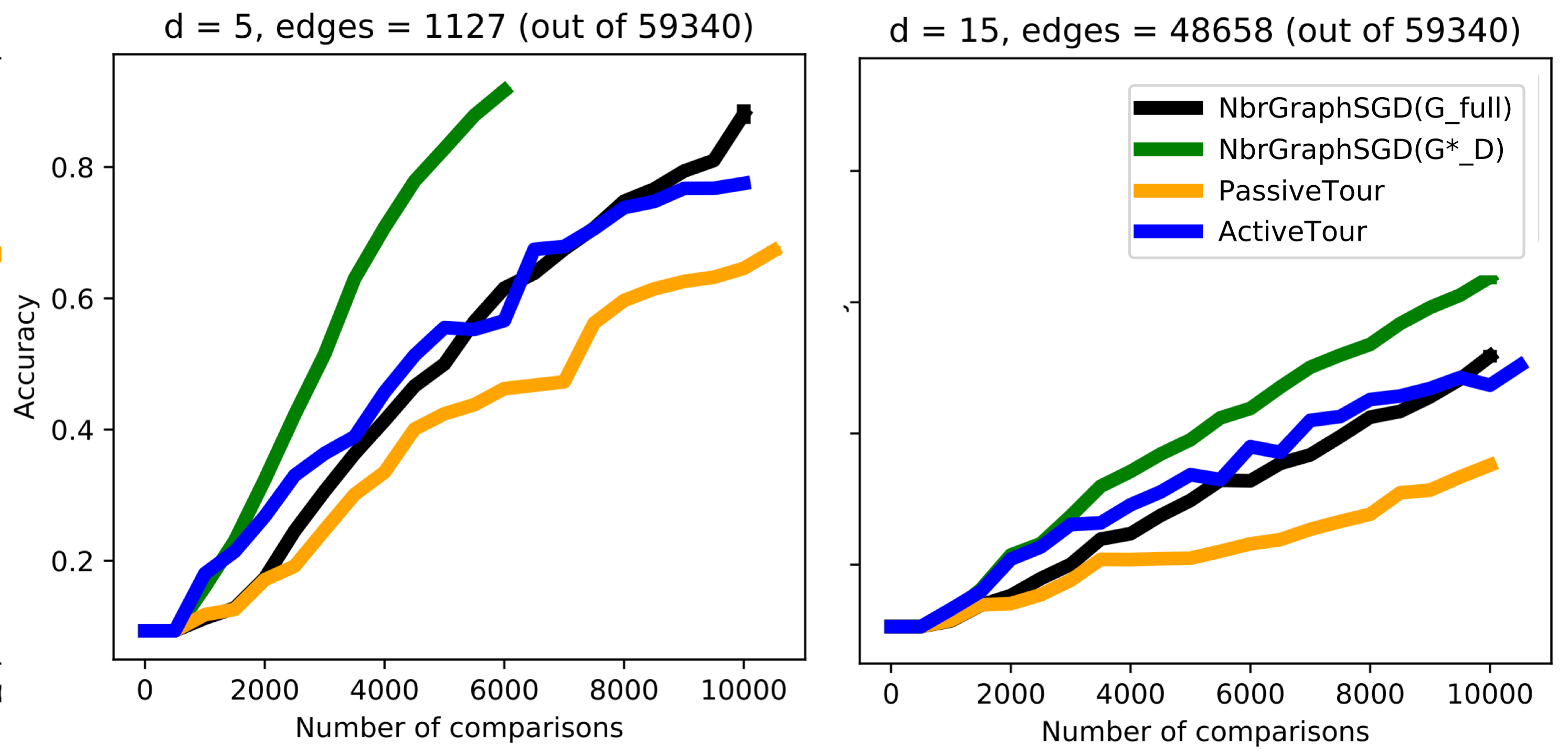}
    \caption{Comparing the performance of $\ALGDalgo$ w.r.t $\GS_\D$ (green) against the baselines on the QuickDraw dataset. In the plot titles, $d$ denotes the dimension of low-dimensional projection of the data and edges is $e(\GS)$, the number of edges in the true neighborhood graph of $\WWS$.}
    %A complete graph on $k=345$ would have $59,340$ edges.}
    \label{fig:quickdraw_random_projection}
\end{figure}

\subsection{Real data}
\label{sec:experiments:real}

The QuickDraw dataset \citep{DBLP:journals/corr/HaE17}, is a collection of 50 million drawings across 345 categories, contributed by players of the game ``Quick, Draw!''. We use the bitmap version of the dataset, which contains these drawings converted from vector format (keystrokes) into 28x28 grayscale images. We randomly select $70,000$ examples from this large data and use 
$60,000$ as our training set and the rest as the test set.
We train a linear teacher on the data after randomly projecting it into $\R^d$. We then use the resulting model $\WWS$ to implement the label-comparison oracle (see Definition \ref{def:oracles}). We
denote the true graph of $\WWS$ (Definition \ref{def:neighborhood}) as $\GS$ and the empirical graph of $\WWS$ (Definition \ref{def:emp-neighborhood}, as computed w.r.t the training set) as $\GS_\D$. 
 
We begin by comparing the performance of $\ALGDalgo$ w.r.t $\GS_\D$ with the same baselines from Section \ref{sec:experiments:synthetic}. We explore the relationship between the  sparsity of the \emph{true graph} $\GS$  
and the performance of $\ALGDalgo$ w.r.t $\GS_\D$ as a function of the dimension $d$ and $k=345$. In Figure \ref{fig:quickdraw_random_projection} we report the query complexities w.r.t $d=5$ %\gale{Is it $d=5$ as noted in the graph itself?} \gy{yes! thanks}
(left) and $d=15$ (right). Note that this is a realizable learning task since the models are measured in terms of their accuracy w.r.t the teacher's predictions, and the teacher is also a linear model. In line with our theoretical results from Section \ref{sec:active}, we observe that the gain from using our method (over e.g. the passive or active tournament baselines) is smaller when the true neighborhood graph is less sparse.

% Next, we address the question of obtaining $\GS_\D$. Note that this graph essentially conveys information about $\Pr_x\sim \D[\fs(x)]$ (specifically,  which pairs of classes are likely to appear together, i.e. as first and second best under $\fs$). This can be thought of as an extension of learning with label proportions \pcite{}, where it is assumed that the frequency of each class is known. In other words, the neighborhood graph can be thought of as conveying prior knowledge about the structure of the classes. To check this intuition we construct a replacement for $\GS_\D$ that does not require querying the teacher $\WWS$ at all. We define $\Gemb$ such that $\Gemb[i,j]$ is the semantic similarity between the names of classes $i$ and $j$. E.g. the similarity between class 12 (backpack) and class 305 (telephone) is the cosine similarity between the pretrained BERT Word Embeddings of \emph{backpack} and \emph{telephone}. We transform $\Gemb$ into a neighborhood graph by drawing an edge between every class and its $K$ nearest neighbors under $\Gemb$. Figure  reports the results of using this neighborhood graph instead of $\GS_\D$. 

%\begin{figure}
%    \centering
%    \includegraphics{nn_visualization.png}
%    \caption{Caption}
%    \label{fig:nn_visualization}
%\end{figure}

\section{Conclusions}
\gale{Remind the reader that we are still focused on the standard supervised objective. We might want to also recall that in a few other preambles to sections and not just mention this in the intro}
We studied the setting where annotators are asked to provide only pairwise label comparisons. We believe this is a natural setting as it is both easy for humans to provide, and still results in sufficient information for learning. \gale{Do we want to make this stronger and explicitly remind the reader that giving the argmax label can be prohibitive?} Our results provide several key characterizations of how this information should be gathered and used. We show that, perhaps counter-intuitively, there are cases 
for which having all the class comparisons per training point does not yield a 
 sample complexity advantage over just receiving the one true class label. On the other hand, in the active setting, we show that comparisons can be used in an effective way that goes beyond obtaining the argmax training labels. \gale{This is a negative wording and we should emphasize the fact that label comparisons are in face beneficial in this case. Also, finish with one sentence about demonstrated advantage in synthetic and real data}

Many interesting open questions remain. First, our focus was on linear classification, and it would be interesting to generalize the result to other classes (such as neural networks). Second, one can consider a mixture of comparisons and true-labels, since the latter may be easy to obtain in some instances, and hence query-complexity should count these cases differently. Finally, here we assumed that annotators can provide answers to all queries. In practice, some queries may not be answerable (e.g., labels are too ``close'' or both are equally bad), and it would be interesting to extend the formalism and practical algorithm to these cases.

%\begin{contributions} % will be removed in pdf for initial submission,
                      % so you can already fill it to test with the
                      % ‘accepted’ class option
    %Briefly list author contributions.
    %This is a nice way of making clear who did what and to give proper credit.
%\end{contributions}

\begin{acknowledgements} Gal Yona is supported by the Israeli Council for Higher Education (CHE) via the Weizmann Data Science Research Center, by a research grant from the Estate of Tully and Michele Plesser, and by a Google PhD fellowship, and this work was done during an internship at Google.  Shay Moran is a Robert J.\ Shillman Fellow, his research is supported in part by the Israel Science Foundation (grant No.\ 1225/20), by a grant from the United States - Israel Binational Science Foundation (BSF), by an Azrieli Faculty Fellowship, by Israel PBC-VATAT,  and by the Technion Center for Machine Learning and Intelligent Systems (MLIS). 

The authors wish to thank Ami Wiesel for contributing many ideas throughout the development of this work and for helpful feedback on this manuscript.
\end{acknowledgements}

\bibliography{uai2022-template}
\clearpage

\onecolumn
\appendix

\section{Passive learning lower bound}
\label{supp:passive_lower_bound}

\begin{theorem}
Let $\Hnab$ be a class whose Daniely-Shwartz dimension is $d$. Then, any algorithm that PAC learns $\Hnab$ using label-comparisons (Definition \ref{def:passive_comparisons}) must use $\Omega(d/\eps)$ samples in the worst case.
\end{theorem}

\begin{proof}
Fix a class $\Hnab$ with Daniely-Shwartz dimension $d$. Let $Z = \set{x_1, \dots, x_d}$ be a set that is shattered by $\Hnab$, and let $\F$ denote the subset of $\Hnab$ whose existence is guaranteed by the definition of shattering\amirg{say more explicitly that this is the subset that shatters $Z$.}\gy{I don't think this will be clear, since it is not called that way in the definition.}. Our objective is to prove that there is a distribution $\D$ on $\X$ and a target concept $\hsnab$ that requires at least $d/\eps$ samples to learn. To do so, we will construct a distribution $\D$ and label it according to a target concept $\hsnab$ that is selected uniformly at random from $\F$. We will show that the expected error is high over the choice of $\hsnab$; this will imply that there is some fixed $\hsnab$ that also leads to high error.

At a high level, the distribution $\D$ we construct is similar to the ``hard'' distribution from the standard  PAC lower bound  (i.e., without label comparisons. See for example \citet{shalev2014understanding}). It puts a relatively large probability mass on one of the $d$ points in $Z$, say $x_{1}$, and splits the remaining probability on $x_{2},\dots,x_{d}$ uniformly, such that any algorithm that takes only $m$ samples from $\D$ will not even encounter a large fraction of $x_{2},\dots,x_{d}$. It is left to argue that the learner cannot do well on the ``unseen'' examples. In our case, we will show this follows directly by the definition of shattering used in the DS dimension. Intuitively, the learner cannot tell whether the label of an unseen example $x$ is given by $f_\nabla$ (in our case, the true target concept $\hsnab$) or by $g_{\nabla}$, where $f_\nabla, g_{\nabla}$ are $x$-close w.r.t $Z$ (Definition \ref{def:close}). Since these two concepts assign a different label to $x$, the learner cannot predict the label of $x$ w.p larger than $0.5$. In total, since there are many such “unseen” examples, the algorithm will incur a large error.

Formally, let $m=\frac{d-1}{64\epsilon}$ and let $A$ be any learning algorithm that observes a sample $S$ of at most $m$ i.i.d samples with label-comparison supervision before picking a hypothesis $h \equiv A(S)$.  Let $Z'=\{x_{2},\dots,x_{d}\}$, and define a distribution $\D$ on $Z$ with point mass $p(\cdot)$ defined as follows:
\begin{equation*}
 p(x)=\begin{cases}
1-16\epsilon & x\notin Z'\\
\frac{16\epsilon}{d-1} & x\in Z' \end{cases}  
\end{equation*}

To simplify, denote \amirg{shouldn't $\nabla$ be on the LHS as well? Perhaps just call it $h$ throughout.}\gy{Moved to using $h$ throughout.}
\begin{align*}
    \text{err}(h)&=\Pr_{x\sim \D}[\arg\max h(x)\neq\arg\max \hsnab(x)] \\
    \text{err}'(h)&=\Pr_{x\sim \D}[\arg\max h(x)\neq\arg\max \hsnab(x)\text{ and \ensuremath{x\in Z'}}]
\end{align*}

Since $\text{err}'(h)$ is a lower bound on $\text{err}(h)$, it suffices to show that $\text{err}'(h)$ is large. In particular, we will show that there exists a target concept such that $\Pr[\text{err}'(h) < \eps] > 1/15$.

Define the event $B(S)$: The sample $S$ contains less than $(d-1)/2$ points from $Z'$. Then it holds that:
\begin{equation}
\label{eqn:lower-bound-1}    
\Pr_{S\sim \D^m}[B(S)] \geq 0.5
\end{equation}

To see this, let the random variable $R$ denote the number of points in $S$ sampled from $\D^m$ that are in $Z'$. By the definition of $\D$, $\E[R]=m\cdot \frac{16\eps}{d-1} = 0.25\epsilon$, and by Markov's inequality, $\Pr[R>(d-1)/2] \leq 0.5$, so we have $\Pr_S[B(S)] \geq 1-\Pr[R>(d-1)/2] \geq 0.5$.

Next, we claim that $\E_{\hsnab,S}[\text{err}'(h)\vert B(S)]>4\epsilon$. Indeed, when $B(S)$ holds, $A$ has not seen at least $(d-1)/2$ of the points in $Z'$. For each of these points, the probability (over the choice of the random $\hsnab$ and $S\sim \D^m$) that the learner correctly predicts the label of $x$ cannot exceed 0.5. 
This is because by the definition of shattering, for every target concept $f_{\nabla}\in\F$ (say, such that $\arg\max f_{\nabla}(x)=\arg\max h_{\nabla}(x)$), there also exists an equally likely target concept $g_{\nabla}\in\F$ that is identical to $f_{\nabla}$ on $Z-\{x\}$ but induces a different label for $x$ (so $\arg\max g_{\nabla}(x)\neq\arg\max h_{\nabla}(x)$). Overall, we get

\begin{equation}
\label{eqn:lower-bound-2}
   \E_{\hsnab,S}[\text{err}'(h)\vert B(S)]>\frac{d-1}{2}\cdot\frac{1}{2}\cdot\frac{16\epsilon}{d-1}=4\epsilon 
\end{equation}

Combining Equations (\ref{eqn:lower-bound-1}) and (\ref{eqn:lower-bound-2}), we have $\E_{\hsnab, S}[\text{err}'(h)]>2\epsilon$.  In particular, there is some $\tilde{h} \in F$ such that $\E_{S}[\text{err}'(h)]>2\epsilon$\amirg{should't it be $h^*$ here? Maybe just call it $\tilde{h}$ to avoid confusion.} (when the error is computed as disagreement with the label of $\tilde{h}$). Fix this as the target concept, and let $p=\Pr_{S}[\text{err}'(h)>\epsilon]$\amirg{no need to define $p$}\gy{I just need it because it appears in the following equation (otherwise it's too long/not clear, I think)}. Note that by definition, $\text{err}'(h)\leq16\epsilon$ (it is only penalized on the mistakes on $Z'$), so

\begin{equation*}
    2\epsilon<\E_{\hsnab,S}[\text{err}'(h)]\leq16\epsilon\cdot \Pr_{S}[\text{err}'(h)>\epsilon]+\epsilon\cdot \Pr[\text{err}'(h)\leq\epsilon]=16\epsilon p+(1-p)\epsilon
\end{equation*}

From this, we have $p > 1/15$. To summarize, we demonstrated a distribution $\D$ and a concept $\hsnab \in \H_{\nabla}$ such that for any learning algorithm $A$ that uses $m=\frac{d-1}{64\eps}$ samples, $\Pr_S[\text{err}(A(S)) > \eps] > 1/15$, as required.

%We need to show a distribution D on Z and a concept h_{\nabla}^{*}\in H_{\nabla} for labeling D, such that w.p at least 7/8, Pr_{x\sim D}[\arg\max h_{\nabla}(x)\neq\arg\max h_{\nabla}^{*}(x)]\leq\epsilon.
\end{proof}

\section{Passive Learning Empirical Evaluation}
\label{supp:passive_experiments}

\begin{figure}[h]
    \centering
    \includegraphics[width=0.85\linewidth]{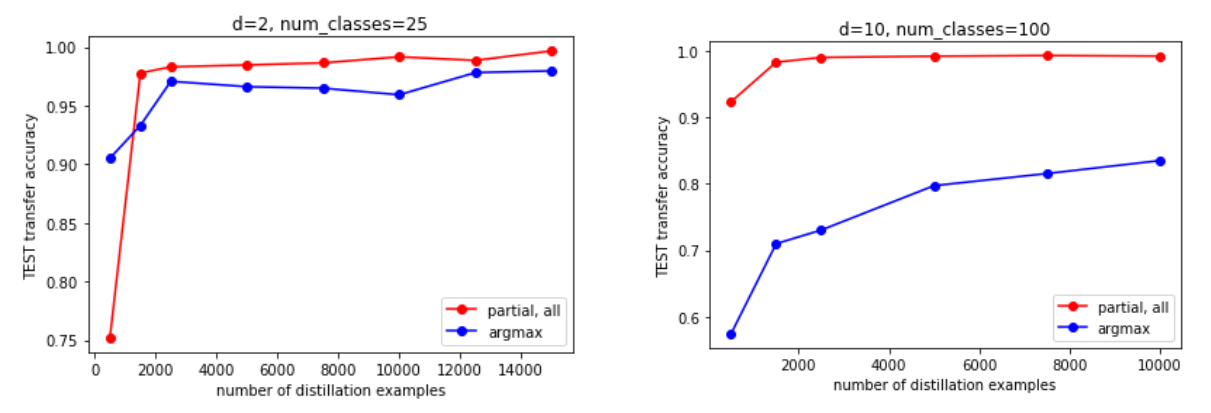}
    \caption{Comparing the transfer accuracy (y-axis) vs number of examples (x-axis) for different values of $d$ and two learners: one that has access to all the label-comparisons (red) and one that only has access to argmax supervision (blue). Results are averaged over $T=10$ independent trials.}
    \label{fig:passive_experiments}
\end{figure}

Recall that in Theorem \ref{lemma:passive}, we proved that the \emph{best} algorithm that uses all pairwise label-comparisons has no advantage (in terms of sample complexity in the PAC sense) over the \emph{best} algorithm that uses argmax labels, for learning linear classifiers in $d=2$. This suggests several natural questions. First, what happens for particular, practical algorithms, and second, to what extent this also holds for $d \gg 1$. 

To examine both questions, we consider the following empirical setup: we initialize a random linear teacher and draw $m$ random examples from a uniform distribution over $\R^d$. As our learning algorithm, we consider a simple gradient-descent algorithm that takes gradient steps to minimize the loss on pairwise label comparisons (see Algorithm \ref{fig:oursgdalg}). Thus, we compare between (i) argmax supervision, in which the label comparisons used to compute the loss $L$ are comparisons involving the correct label $y$ against every other label, (ii) full pairwise supervision, in which the loss $L$ is computed w.r.t all ${k \choose 2}$ pairwise label comparisons. See Figure \ref{fig:passive_experiments} for sample complexity plots comparing both algorithms, for $d=2$ (left) and $d=10$ (right). We see that for uniform data, having access to all the label-comparisons provides no gains when $d=2$ (as our negative result suggests), but does provide considerable gains when $d=10$. As discussed in Section \ref{sec:passive}, we conjecture that this is an artifact of the fact that the distribution is uniform -- and that under structural assumptions on $\D$ (e.g. sparsity or margin), the results for $d \gg 2$ will look similar to those for $d=2$.

\section{Active Learning Complexity for $d>1$} \label{supp:d_gt1}
Our analysis in the main text suggested that when we can efficiently learn $\GS$ and it is sparse, label-comparisons provide a gain over argmax queries. We showed this  when $d=1$, and it is natural to ask to what happens for $d>1$. This requires addressing both the question of what is the binary active learning primitive that we use, as well as the questions of sparsity and learning the graph. We discuss these two points below.

Regarding binary active learning, for $d>1$, active learning is known to no longer provide asymptotic benefits over passive learning in the distribution-free setting \citep{dasgupta2004analysis}. Hence for $d>1$ it's natural to focus on \emph{distribution-specific} learning. Namely, we will understand $q_b(\cdot)$ to be the query complexity of active learning the binary class $\H^{2,d}_{\linear}$ under structural assumptions on the target distribution. For example, \cite{balcan2013active} prove that for log-concave distributions, $q_b(\gamma) = \Theta(d \log \frac{1}{\gamma})$. 

\amirg{Here when we talk about $d=1$, it is with bias. Is this how we framed it above. Because initially we presented it as homogeneous.} 
Regarding sparsity, we note that linear classifiers in $d=1$ are maximally spase.
%\footnote{every vertex that corresponds to an ``effective'' class (a class $i$ for which there exists some $\xx \in \R^d$ for which $i$ is its label) must have at least one neighbor by definition. Thus, when every class  is an ``effective'' class,  $\deg(G) = \Omega(k)$.}
Interestingly, linear classifiers in $d=2$ also admit sparse graphs. By definition, the decision regions of $ h(\cdot; \WW) \in \H^{k,3}_{\linear}$
are convex polyhedrons. By Steinitz's theorem, these are 3-connected planar graphs, and so a corollary of Euler’s Formula implies that $\deg(G)\leq3k-6$. In Section \ref{sec:experiments:synthetic} we experimentally evaluate the sparsity of random linear teachers as a function of $k,d$, observing that the graphs are becoming sparse as $k \gg d$. This is related to a known result in computational geometry \cite{dwyer1991higher}, that the expected number of edges in a Voronoi regions is linear, when points are sampled uniformly from the sphere (see also \cite{aurenhammer2000voronoi}).
%(but we conjecture that the dependence on the dimension could be exponential). \gy{It would be good if we would have something to reference here.}
We leave open the question of whether the graph learning procedure of Algorithm \ref{algo:learning_G_1d} has an efficient analogue in $d>1$. 

\section{Demonstration of $G^*_D$ misses}

In the main text, we discuss the empirical neighborhood graph $G^*_D$, which is a subset of the true graph $G^*$. In Figure \ref{fig:G_D_example} below we show an example where $G^*_D$ will result in erroneous classification, because it will miss important boundaries points.

\begin{figure}[h]
    \centering
    \includegraphics[width=0.75\linewidth]{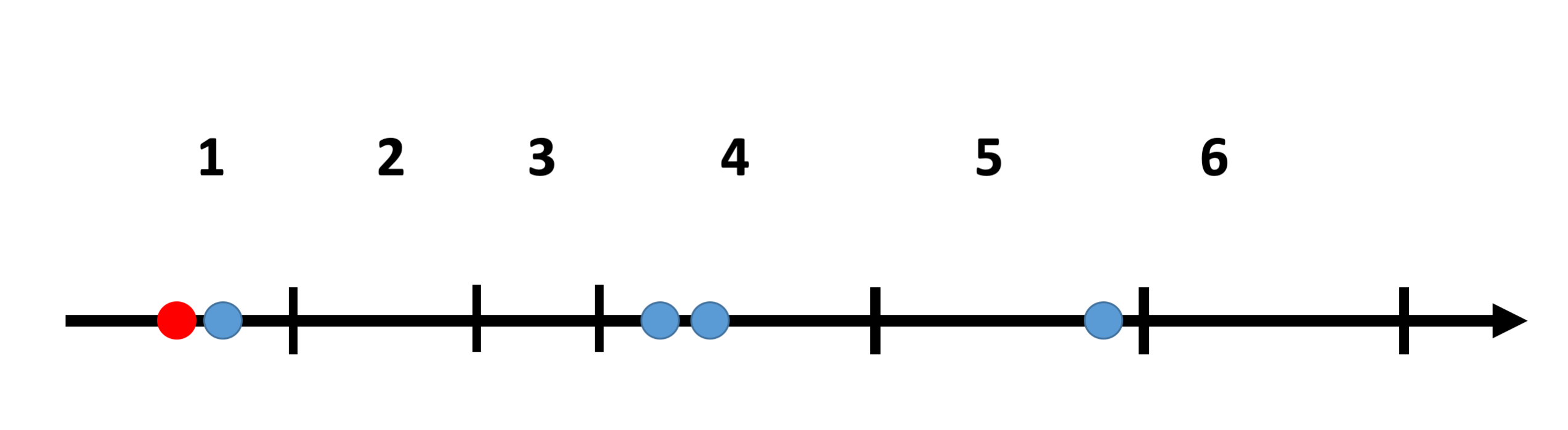}
    \caption{For a linear classifier in $d=1$ and distribution on $\R$ given by the blue circles, the empirical graph $\GS_\D$ only has the edges $(1,2), (3,4)$ and $(5,6)$. Thus the aggregated classifier $f^{(\GS_\D, C^\star)}$ may incorrectly classify the red circle as belonging to class $3$ since it also ``beats'' all its' neighbor classes. Here, the culprit is that $\GS_\D$ misses the edge $(2,3) \in \GS$, as it was not witnessed by $x$ in the support of $\D$. }
    \label{fig:G_D_example}
\end{figure}

\section{Proof of Lemma \ref{lemma:aggregation} }
\label{sec:aggregation_proof}
%\begin{proof}
Fix $\D$ and $\WWS$. We will use an important observation: if $f^{\star}$ is $f^{(G,C)}$ w.r.t  $G=G(\WWS)$ and $C$ given by the \emph{true} binary classifiers (i.e. $h_{ij}=\WWS_{i}-\WWS_{j} = h^\star_{ij}$), then $L_{\D}(f^\star) = 0$. I.e, for every 
 $\xx\in\R^{d}$ with argmax $y=\arg\max_{i\in[k]}\WWS_{i}\xx$, the following holds: $\forall j:\quad 1=f_y^{\star}(\xx)>f_j^{\star}(\xx)$.

Now, fix $(G,C)$ satisfying the conditions in the theorem statement.  We can upper bound $L(f^{(G,C)})$ as follows: %\gale{Why introduce the $\Join$ additional notation?}\gy{it was just to make it shorted so it won't overflow, but when it's on the left side it overflows anyway. will fix}
\begin{align*}
    L(f^{(G,C)})	&=\E_{x\sim\D}[f^{(G,C)}(\xx)\neq\arg\max_{i\in[k]}\WWS_{i}\xx] \\
	&=\Pr_{x\sim\D}[\exists j\in[k]:\,\,f^{(G,C)}(\xx)_{y}<f^{(G,C)}(\xx)_{j}] \\
	&\leq\Pr_{x\sim\D}[\exists(i,j)\in G:\,\sign(h_{ij}(\xx))\neq \sign(h^\star_{ij}(\xx))] \\
	&\leq\sum_{(i,j)\in G}\Pr_{x\sim\D}[\sign(h_{ij}(\xx)) \neq \sign(h^\star_{ij}(\xx)] \\
	&\leq\sum_{(i,j)\in G}\epsilon/\deg(G) \leq \eps
\end{align*} 
As required.
%\end{proof}

% \section{Algorithm $\ALGDalgo$}

% \begin{algorithm}[h]
%   \caption{\textbf{$\ALGDalgo$ }}
%   \label{alg:TMC-Shapley}
%       \begin{algorithmic}
%         \STATE {\bfseries Input:}  Label neighborhood graph $G$, buffer size $R$, steps $T$, confidence parameter $\tau$, learning rate $\eta$.
%         \STATE {\bfseries Output:} classifier $ h(\cdot; \WW): \R^d \to \R^k$, number of comparisons $q$.
%         \vspace{1mm}
%         \STATE Initialize $\WW^{(0)}$ , $L=0$, $q=0$, $b=0$.
%         \FOR{$t = 1, 2, \dots, T$}
%         \STATE Sample $\xx \sim \D$.
%         \STATE Sample $(i,j)$ uniformly from the edges of $G$.
        
%         \IF{ $\card{\WW^{(t-1)}(\xx)_i - \WW^{(t-1)}(\xx)_j} < \tau$} 
%         \STATE $L \pluseq L(\xx, i,j, \WW^{(t-1)})$
%         \STATE $q \pluseq 1$, $b \pluseq 1$.
%         \ENDIF

%         \IF{ $b \geq r$}
%         \STATE Update $\WW^{(t)} \leftarrow \WW^{(t-1)} - \eta\cdot \frac{\partial L}{\partial \WW}$
%         \STATE Clear buffer: $L = 0$, $b=0$. 
%         \ENDIF
%         \ENDFOR
%         \vspace{1mm}
%       \end{algorithmic}
% \end{algorithm}

\end{document}
\end{document}